\documentclass[11pt,a4paper]{article}

\usepackage[utf8]{inputenc} 
\usepackage[T1]{fontenc} 
\usepackage{amsmath,amsthm,latexsym} 
\usepackage[fixamsmath]{mathtools}
\mathtoolsset{showmanualtags,mathic,centercolon} 
\usepackage{amssymb,amsfonts} 
\usepackage{enumitem}
\usepackage{algpseudocode} 
\usepackage{algorithm} 
\usepackage{bm} 
\usepackage{url}
\usepackage{authblk} 
\usepackage[margin=1.25in]{geometry} 

\numberwithin{equation}{section} 
\theoremstyle{plain}
\newtheorem{theorem}{Theorem}[section] 
\newtheorem{lemma}[theorem]{Lemma} 
\newtheorem{corollary}[theorem]{Corollary}

\theoremstyle{definition}
\newtheorem{definition}[theorem]{Definition}

\newtheorem{assumption}[theorem]{Assumption}

\theoremstyle{remark}
\newtheorem{remark}[theorem]{Remark}

\newcommand{\secref}[1]{Section~\ref{#1}}

\newcommand{\thmref}[1]{Theorem~\ref{#1}}
\newcommand{\lemref}[1]{Lemma~\ref{#1}}
\newcommand{\corref}[1]{Corollary~\ref{#1}}
\renewcommand{\algref}[1]{Algorithm~\ref{#1}}
\newcommand{\assref}[1]{Assumption~\ref{#1}}

\newcommand{\R}{\mathbb{R}} 
\newcommand{\bigO}{\mathcal{O}} 
\DeclareMathOperator*{\argmin}{arg\,min} 
\newcommand{\grad}{\nabla} 

\newcommand{\Agg}{\text{\sc Agg}}

\def\be{\begin{equation}}
\def\ee{\end{equation}}

\algrenewcommand\algorithmicrequire{\textbf{Input:}}
\algrenewcommand\algorithmicensure{\textbf{Output:}}

\newcommand{\revision}[1]{{#1}}  

\begin{document}
\title{A Simplified Convergence Theory for Byzantine Resilient Stochastic Gradient Descent\thanks{This work has been published in the \textit{EURO Journal on Computational Optimization}. See \texttt{https://doi.org/10.1016/j.ejco.2022.100038} for the final version.}}
\author[1,2]{Lindon Roberts}
\author[1]{Edward Smyth}
\affil[1]{\small Mathematical Sciences Institute, Building 145, Science Road, Australian National University, Canberra ACT 2601, Australia (\texttt{first.last@anu.edu.au}).}
\affil[2]{\small School of Mathematics and Statistics, Carslaw Building, University of Sydney, Camperdown NSW 2006, Australia (\texttt{lindon.roberts@sydney.edu.au}).}

\date{\today}
\maketitle

\begin{abstract}
	In \revision{distributed} learning, a central server trains a model according to updates provided by \revision{nodes} holding local data samples. In the presence of one or more malicious servers sending incorrect information (a Byzantine adversary), standard algorithms for model training such as stochastic gradient descent (SGD) fail to converge. 
	In this paper, we present a simplified convergence theory for the generic Byzantine Resilient SGD method originally proposed by Blanchard et al.~[NeurIPS 2017]. \nocite{blanchard2017byzantinetolerant}
	Compared to the existing analysis, we shown convergence to a stationary point in expectation under standard assumptions on the (possibly nonconvex) objective function and flexible assumptions on the stochastic gradients.
\end{abstract}

\\


\section{Introduction} \label{sec_introduction}
\revision{Distributed learning is a branch of machine learning in which large datasets are stored on several client devices, and a central server controls the learning process \cite{NIPS2010_abea47ba,MLSYS2019_c74d97b0}. 
In this setting, the local data is never shared or transferred across clients, or shared with the central server, and the actual optimization can be performed either by the central server or separately on each client \cite{NIPS2006_77ee3bc5,45187}. 
In the first case, learning is performed by the central server requesting (stochastic) gradient information from each client and aggregating this information within its optimization procedure \cite{10.1109/ALLERTON.2019.8919791,JMLR:v11:teo10a}.

From a mathematical optimization perspective, distributed learning presents several challenges including communication overheads \cite{46622,Dai2021} and asynchronicity \cite{10.1109/ALLERTON.2019.8919791}.
This may be further complicated in the case of federated learning, where the client datasets are assumed to be heterogeneous \cite{wang2021field}.}
Here, we consider the case of \revision{distributed} learning in the presence of adversarial attacks; the reliance on many clients and network communication makes this a more relevant concern than in traditional learning.
Specifically, we consider the case of Byzantine adversaries, which alter the gradient information sent from (potentially multiple) clients to the server, for example through data poisoning \cite{Biggio2012,chen2017targeted}, where individual clients' datasets are altered, or model update poisoning \cite{bhagoji2019analyzing}, where the information sent to the central server is directly corrupted.
The goal of adversarial attacks can be untargeted, attempting to decrease the model's accuracy overall \cite{Biggio2012}, or targeted, aiming to change the model's behavior on only a specific subset of inputs \cite{chen2017targeted,bhagoji2019analyzing}.
A comprehensive taxonomy of adversarial attacks is given in \cite[Section 5]{kairouz2021advances}.

In the presence of adversarial attacks, standard optimization methods such as stochastic gradient descent (SGD) can fail \cite{blanchard2017byzantinetolerant}, and so the development of optimization algorithms which are Byzantine resilient is of critical importance to the success of \revision{distributed} learning.
Where the individual clients' datasets are i.i.d., algorithms such as SGD can be \revision{made} Byzantine resilient using robust averaging techniques.
In this paper, we introduce a simplified convergence analysis of Byzantine Resilient SGD (BRSGD) \cite{blanchard2017byzantinetolerant} for nonconvex learning problems.
Our analysis covers the same range of approaches for aggregating gradient information from i.i.d.~client datasets, but uses more standard smoothness assumptions on the objective function and the non-corrupted stochastic gradients.
As a trade-off, we show convergence to first-order optimality in expectation, rather than almost surely as in \cite{blanchard2017byzantinetolerant}.
\revision{We also show a convergence rate of $\bigO(1/K^{(1-p)/2})$ when the learning rate sequence is chosen as $\bigO(k^{-p})$ for $p\in(1/2,1)$, which was not provided in \cite{blanchard2017byzantinetolerant}.}

\paragraph{Related Work}
A general analysis of SGD in a federated learning setting \revision{(which generalizes distributed learning)} for nonconvex objectives can be found in \cite{li2020unified}.
In the case of Byzantine resilient SGD for i.i.d.~client datasets, several different robust averaging techniques have been proposed.
These include the geometric median \cite{Chen2017,xie2018generalized,pillutla2019robust}, coordinate-wise medians and trimmed means \cite{pmlr-v80-yin18a,pmlr-v97-yin19a}, neighborhood-based averaging \cite{blanchard2017byzantinetolerant}, iterative filtering \cite{su2019securing,pmlr-v97-yin19a} and combinations of multiple such approaches \cite{pmlr-v80-mhamdi18a}.

The types of convergence theory of these different methods varies.
Convergence only for strongly convex objectives is considered in \cite{Chen2017,pillutla2019robust,pmlr-v80-yin18a}, and in \cite{su2019securing} for the full population loss (rather than the empirical loss).
For nonconvex objectives, high probability convergence to approximate first- and second-order optimal points is given in \cite{pmlr-v80-yin18a} and \cite{pmlr-v97-yin19a} respectively.

Alternatively, \cite{blanchard2017byzantinetolerant} considers BRSGD applied to nonconvex objectives with decreasing learning rates, which proves almost sure convergence to stationary points rather than to a neighborhood, similar to the standard SGD setting \cite{bottou2018optimization}.
This analysis applies to a generic robust aggregator satisfying specific assumptions.
Different aggregations methods which satisfy this assumption have been proposed in \cite{blanchard2017byzantinetolerant,xie2018generalized,pmlr-v80-mhamdi18a}.

We conclude by noting that \cite{Xie2019,data2020byzantineresilient,pmlr-v139-data21a} consider the case of Byzantine resilient federated learning in the case of non-i.i.d.~datasets and have developed specific robust aggregators suited for this setting (with associated convergence theory).
This is a more difficult problem, made visible for instance by \cite{data2020byzantineresilient} requiring a maximum of 25\% of clients be corrupted (rather than approximately 50\% for the i.i.d.~case).

\paragraph{Contributions}
In this paper, we present a simplified convergence result for BRSGD. 
Unlike the original analysis in \cite{blanchard2017byzantinetolerant} (which was based on the analysis of SGD in \cite{Bottou98}), we use standard smoothness assumptions on the objective, closer in spirit to the standard analysis of SGD \cite{bottou2018optimization}.
For the (non-corrupted) stochastic gradient estimates, we use a general expected smoothness assumption based on \cite{khaled2020}.
Under these conditions, we prove the convergence of BRSGD to stationary points in expectation: we note that this is weaker than the almost-sure result from \cite{blanchard2017byzantinetolerant}, coming from our more standard problem assumptions and simpler analysis. Our result and proof technique has some similarities to \cite[Lemma 4.3]{sebbouh2021sure}, which shows almost sure convergence of standard SGD under general expected smoothness conditions on stochastic gradients. 

As described above, since BRSGD in \cite{blanchard2017byzantinetolerant} is a generic framework, our results are applicable to any aggregation function satisfying the same assumptions, including all those in \cite{blanchard2017byzantinetolerant,xie2018generalized,pmlr-v80-mhamdi18a}.

\paragraph{Structure}
We begin by describing the general model of \revision{distributed} learning with Byzantine adversaries and describe the BRSGD method in \secref{sec_model}.
Our new convergence analysis is given in \secref{sec_convergence}.
\revision{The corresponding convergence rates are shown in \secref{sec_complexity}}, and we conclude in \secref{sec_conclusion}.

\paragraph{Notation}
We use $\|\cdot\|$ to be the Euclidean norm and $\langle \cdot, \cdot \rangle$ to be the corresponding inner product on $\R^d$, and let $[m]:=\{1,\ldots,m\}$.

\section{Problem \& Byzantine Resilient SGD} \label{sec_model}
We begin by describing the Byzantine adversarial model problem and the Byzantine resilient SGD algorithm from \cite{blanchard2017byzantinetolerant}.

\subsection{Byzantine Adversarial Model Problem}
In the \revision{distributed} learning problem, our data is \revision{split} across $m$ nodes (or devices) while our model is centralized \cite{kairouz2021advances}. For our $i^{\text{th}}$ node, $\{((x_i)_j, (y_i)_j)\}_{j=1}^{n_i}$ is our dataset, where $n_i$ is the number of data-points stored at that node. We will assume that each element is drawn from a distribution that is common across all nodes, $((x_i)_j, (y_i)_j) \sim \Omega$. As this distribution is theoretical, we replace our unknown distribution with the known empirical distribution $\Omega'_i$, where we select each element of our dataset with probability $\frac{1}{n_i}$. 
Having been given this dataset, from a space of functions parametrized by $w \in \mathbb{R}^d$, we wish to find a model function $f$ for which $f(x) \approx y$, for all $i \in [m]$ and $(x,y) \sim \Omega'_i$. Hence, we look for a function $f(\cdot;w)$ that minimizes the average empirical risk across our $m$ nodes: 
\begin{equation} 
    \min_{w \in \mathbb{R}^d} F(w) := \frac{1}{m} \sum^m_{i=1} F_i(w), \label{prob:2}
\end{equation}
where $F_i$ is the empirical risk of the model $f(w)$ on the $i^{\text{th}}$ node. Specifically: 
\begin{equation}
    F_i(w) :=  \frac{1}{n_i} \sum^{n_i}_{j=1} l((y_i)_j,f((x_i)_j;w)).
\end{equation}
where $l(\cdot,\cdot)$ is a loss function that quantifies the difference between the two values. Note that both the loss function and the model remain constant across the nodes. 

We solve \eqref{prob:2} with iterative methods converging to a neighbourhood of stationary points of our problems.
These iterative methods involve each node sending an estimation of the gradient of our function at the current point to the central server.  We model component failure or corruption by setting some of our nodes to be Byzantine adversaries \cite{blanchard2017byzantinetolerant}, who may send arbitrary values to the central server. 
In our problem, of our $m$ nodes, $q$ will be Byzantine adversaries as defined below. Typically we require $q < m/2$, however there are slight differences in some methods.

\begin{definition}[Byzantine Adversaries]
Let $\{(g_1)_k, (g_2)_k, ..., (g_m)_k\}$ be the set of correct local gradient estimators calculated by each node for the $k^{\text{th}}$ iteration of an algorithm. If $q$ out of $m$ vectors are Byzantine adversaries, then, the set of correct vectors at every iteration are partially replaced by vectors $\{(\Tilde{g}_1)_k, (\Tilde{g}_2)_k, ..., (\Tilde{g}_m)_k\}$, according to:
\begin{equation}
    (\Tilde{g}_i)_k :=     
    \begin{cases}
    (B_i)_k,& \text{if}\ \text{the $i^{\text{th}}$ node is Byzantine,}\\
    (g_i)_k,    & \text{otherwise.}
    \end{cases}
\end{equation}
The indices of the adversaries may change across iterations and the value of the Byzantine gradient, $(B_i)_k$, may be a function dependent on $\{(g_1)_k, (g_2)_k, ..., (g_m)_k\}$, the current value of the model $x_k$, the current step-size $\alpha_k$, or any previous information.
\end{definition}

In the taxonomy of \cite[Section 5]{kairouz2021advances}, this framework corresponds to dynamic, white box adversaries using within-update collusion, where the adversary functions via data poisoning or model update poisoning.

\subsection{BRSGD Algorithm}
We now outline the Byzantine Resilient SGD (BRSGD) algorithm framework for solving \eqref{prob:2} in the presence of Byzantine adversaries. 
In iteration $k$ of this method, initially from \cite{blanchard2017byzantinetolerant}, the central server initially collects the (possibly corrupted) local gradient information $(\Tilde{g}_i)_k \in \R^d$ from each node $i \in [m]$.
It then aggregates these $m$ vectors to produce a final gradient estimate $A_k\in\R^d$ using some aggregation function $\Agg\left((\Tilde{g}_1)_k, (\Tilde{g}_2)_k, ..., (\Tilde{g}_m)_k\right)$: we will later give specific requirements on $\Agg$.
We then take a gradient descent-type step with a pre-specified learning rate $\alpha_k>0$. 
The full BRSGD method is given in \algref{alg_brsgd}.

\begin{algorithm}[tb]
	\begin{algorithmic}[1]
		\Require Initial point $w_0 \in \mathbb{R}^d$, sequence of learning rates $\{\alpha_k\}^{\infty}_{k=0}$ with $\alpha_k > 0$ for all $k$, gradient aggregation function $\Agg: \underbrace{\R^d \times \cdots \times \R^d}_{\text{$m$ times}} \to \R^d$.
		\vspace{0.5em}
		\For{$k=0,1,2,\ldots$}
			\For{machines $i = 1,2,...,m$ in parallel}
				\If{$i$ is a non-Byzantine node}
					\State Compute local stochastic gradient\ $(\Tilde{g}_i)_k = (g_i)_k$, where $(g_i)_k \in \mathbb{R}^d$ is the (stochastic) gradient estimate at node $i$.
				\Else
					\State Set $(\Tilde{g}_i)_k = (B_i)_k$, where $(B_i)_k \in \mathbb{R}^d$ is the attack gradient.
				\EndIf
			\EndFor
		    \State Aggregate received gradient information: $A_k = \Agg\left((\Tilde{g}_1)_k, (\Tilde{g}_2)_k, ..., (\Tilde{g}_m)_k\right)$.
		    \State Set $w_{k+1} = w_k - \alpha_k A_k$.
		\EndFor
	\end{algorithmic}
	\caption{Byzantine Resilient Stochastic Gradient Descent (BRSGD) \cite{blanchard2017byzantinetolerant}}
	\label{alg_brsgd}
\end{algorithm}

Since the non-corrupted local gradient estimators $(g_i)_k$ can be stochastic gradient estimates, we also outline our formal stochastic approach to this problem.
We introduce a probability space, $(\mathbb{P},\mathcal{F},\Omega)$, with an associated filtration $(\mathcal{F}_k)_{k\in\mathbb{N}}$ to model our gradient information as an adapted process on a filtration, based on framework from \cite{paquette2018}. 
Our sample space will be $\Omega' := \Omega_1' \times \Omega_2' \times... \times \Omega_m'$. Let the random variable modelling the sample drawn at the $i^{th}$ node for the $k^{th}$ iterate be $(\zeta_k)_i := (\zeta_k)_i (\omega)$. 
The value of the gradient estimator at the $i^{th}$ node for the $k^{th}$ iterate will be modelled by $g_i(w_k) := G(w_k,(\zeta_k)_i)$, where $G(w,\zeta)$ is the random variable modelling gradient updates at point $w$ and it will be denoted $G(w)$. 
Since our datasets are assumed to be i.i.d., our function $G(w_k, \zeta)$ is not dependent on the node, and hence, we have the following property: 
\begin{equation}
    \mathbb{E}[g_1(w_k)] = \mathbb{E}[g_2(w_k)] = \cdots = \mathbb{E}[g_m(w_k)].
\end{equation}
Our filtration $\mathcal{F}_k$ will be the $\sigma$-algebra generated by our previous random variables, i.e. $\mathcal{F}_k = \sigma \left(\bigcup^{k-1}_{i=0} \bigcup^{m}_{j=1} \sigma((\zeta_i)_j) \right)$ and let $\mathcal{F} := \bigcup^{\infty}_{n=1} \mathcal{F}_k$. 
Therefore, both our aggregated and individual gradient estimates are an adapted processes on the filtration.  
Furthermore, the iterates of our algorithm will have the  property: 
\begin{equation}
    \mathbb{E}_k[w_{k}] = w_{k},
\end{equation}
where we denote the conditional expectation of a random variable with respect to filtration as $\mathbb{E}_k[\cdot] := \mathbb{E}[\cdot|\mathcal{F}_k]$. 

\paragraph{Choice of aggregation function}
For our convergence theory to hold, the aggregation function $\Agg$ must satisfy the following assumption, from \cite{blanchard2017byzantinetolerant}.

\begin{assumption}[\revision{$\alpha$}-Byzantine Resilience] \label{ass_resilient_agg}
Let $A_k := \Agg\left((\Tilde{g}_1)_k, (\Tilde{g}_2)_k, ..., (\Tilde{g}_m)_k\right)$.
Then there exists a constant \revision{$\alpha\in[0,\pi/2)$} such that for all $k$ we have
\begin{enumerate}
	\item \revision{$\langle \mathbb{E}_k[A_k], \nabla F(w_k) \rangle \geq (1 - \sin(\alpha)) \|\nabla F(w_k)\|^2$}, and 
	\item For $r=2,3,4$, $\mathbb{E}_k[\|A_k\|^r]$ bounded above by a linear combination of terms: \linebreak $\mathbb{E}_k[\|G(w_k)\|^{r_1}] \cdot \mathbb{E}_k[\|G(w_k) \|^{r_2}]\cdot ... \cdot \mathbb{E}_k[\|G(w_k)\|^{r_n}]$ with $r_1 + r_2 + ... + r_n = r$.
\end{enumerate}
\end{assumption}
\revision{
\begin{remark}
	We note that while \cite{blanchard2017byzantinetolerant} requires $r=2,3,4$ in \assref{ass_resilient_agg}(2), our convergence result only requires $r=2$. 
	In fact, our result requires only the weaker condition
	\begin{align}
		\mathbb{E}_k[\|A_k\|^2] \leq E \mathbb{E}_k[\|G(w_k)\|^2], \label{eq_resilient_agg_weaker}
	\end{align}
	for some constant $E$.
	The $r=2$ case of \assref{ass_resilient_agg}(2) implies \eqref{eq_resilient_agg_weaker} via
	\begin{align}
		\mathbb{E}_k[\|A_k\|^2] \leq C_1 \mathbb{E}_k[\|G(w_k)\|^2] + C_2 \mathbb{E}_k[\|G(w_k)\|]^2 \leq (C_1+C_2) \mathbb{E}_k[\|G(w_k)\|^2],
	\end{align}
	for some constants $C_1,C_2>0$, and where the second inequality follows from Jensen's inequality for conditional expectations.
\end{remark}
}

The works \cite{blanchard2017byzantinetolerant,xie2018generalized,pmlr-v80-mhamdi18a} give multiple examples of suitable functions $\Agg$ which satisfy \assref{ass_resilient_agg}, where \revision{$\alpha$} typically depends on $\Agg$, the number of clients $m$ and the number of corrupted clients $q$. 
\revision{Specifically, the $\Agg$ functions proposed in these works are:
\begin{itemize}
	\item Krum \cite{blanchard2017byzantinetolerant}, which returns the vector from $\{(\Tilde{g}_1)_k, (\Tilde{g}_2)_k, ..., (\Tilde{g}_m)_k\}$ that minimizes the distance between it and its nearest neighbours;
	\item Variations of the median \cite{xie2018generalized}, including the marginal (i.e.~component-wise) median, the geometric median
	\begin{align}
		\Agg\left((\Tilde{g}_1)_k, (\Tilde{g}_2)_k, ..., (\Tilde{g}_m)_k\right) = \argmin_{g\in\R^d} \sum_{i=1}^{m} \|g-(\Tilde{g}_i)_k\|,
	\end{align}
	and the `mean-around-median', the (component-wise) mean of the closest vectors to the marginal median;
	\item Bulyan \cite{pmlr-v80-mhamdi18a}, which uses an existing $\Agg$ function such as the above to generate a set of candidate gradients and then apply a `mean-around-median' aggregation to that set.
\end{itemize}
}

\subsection{Existing Convergence Theory}
We now describe the underlying assumptions and state the existing convergence theory for BRSGD (\algref{alg_brsgd}) as given in \cite{blanchard2017byzantinetolerant}.
The requirements on the objective \eqref{prob:2} are based on \cite{Bottou98}:

\begin{assumption} \label{ass_old_smoothness}
The objective function $F$ is $C^3$, bounded below, and there exist constants $D,\epsilon \geq 0$ and $0 \leq \beta < \frac{\pi}{2} $ such that, for all $w \in \mathbb{R}^d$ with $\|w\|^2 \geq D$, we have:
\begin{equation}
    \|\nabla F(w)\| \geq \epsilon, \quad \text{and} \quad  \frac{\langle w, \nabla \revision{F}(w) \rangle}{\|w\| \cdot \|\nabla \revision{F}(w)\|} \geq \cos(\beta). \label{eq_convex_eventually}
\end{equation}
\end{assumption}

Next, the requirement on the stochastic gradients is given by the following two assumptions.

\begin{assumption} \label{ass_old_gradients1}
The gradient estimator $G(w_k)$ is unbiased, $\mathbb{E}_k[G(w_k)] = \nabla F(w_k)$, and for all $r \in \{2,3,4\}$, there exist non-negative constants $A_r$ and $B_r$ such that \linebreak $\mathbb{E}_k[\|G(w_k)\|^r] \leq A_r + B_r \|w_k\|^r$.
\end{assumption}

We are now able to state the main convergence result from \cite{blanchard2017byzantinetolerant}.

\begin{theorem}[Proposition 2, \cite{blanchard2017byzantinetolerant}] \label{thm_old_convergence}
Suppose that Assumptions \ref{ass_resilient_agg}, \ref{ass_old_smoothness} and \ref{ass_old_gradients1} hold, with $\alpha + \beta < \pi/2$ (with $\alpha$ from \revision{\assref{ass_resilient_agg}} and $\beta$ from \assref{ass_old_smoothness}), \revision{and
\begin{align}
	\eta(m,q) \sqrt{\mathbb{E}[\|G(w,\zeta)-\grad F(w)\|^2]} \leq \sin(\alpha)\|\grad F(w)\|, \label{eq_eta_condition}
\end{align}
where $m > 2q+2$ and
\begin{align}
	\eta(m,q) := \sqrt{2\left(m-q+\frac{q(m-q-2) + q^2(m-q-1)}{m-2q-2}\right)}.
\end{align}
} 
\revision{If} we run \algref{alg_brsgd} with learning rates satisfying $\sum_{k=0}^{\infty} \alpha_k = \infty$ and $\sum_{k=0}^{\infty} \alpha_k^2 < \infty$.
Then $\lim_{k\to\infty} \|\grad F(w_k)\| = 0$ almost surely.
\end{theorem}

We conclude by noting the technical condition $\alpha+\beta < \pi/2$, relating the errors in the stochastic gradients and the smoothness of the objective.
We do not need a condition like this in our analysis.
\revision{The condition \eqref{eq_eta_condition} relates the stochastic gradients to the aggregation function (provided there are not too many adversaries). Although our analysis does not require this condition explicitly, similar conditions are required to show Byzantine resilience of aggregation functions such as Krum \cite{blanchard2017byzantinetolerant} and the geometric median \cite{xie2018generalized}.}

\section{New Convergence Analysis} \label{sec_convergence}
We now present our new, simplified analysis of BRSGD.
Compared to \thmref{thm_old_convergence}, our result uses simpler and more common assumptions on both the objective $F$ and the stochastic gradient estimator $G(w,\zeta)$, and a specific choice of learning rate.
As a result, the conclusion is weaker: we get convergence to stationary points in expectation rather than almost surely.
It has the same requirements on the aggregating function $\Agg$.

In particular, we note that \thmref{thm_old_convergence} requires \eqref{eq_convex_eventually}, essentially that $F$ is `convex enough' outside a certain bounded region.
This has the downside, for example, of excluding the model space of neural networks with soft-max activation functions, a common model space in distributed learning, as the activation function has flat asymptotes \cite{Bottou98}. 

\begin{assumption} \label{ass_new_smoothness}
	Each term in the objective function $F$ is $L$-smooth (i.e.~continuously differentiable and $\grad F$ is $L$-Lipschitz continuous) and bounded below by some $F_{\text{low}}$.
\end{assumption}
The key implication of $F$ being $L$-smooth is the upper bound
\begin{equation}
F(y) \leq F(x) + \langle \nabla F(x),y-x \rangle + \frac{L}{2}\|y-x\|^2, \qquad \forall x,y\in\R^d. \label{eq_lipschitz}
\end{equation}

For our stochastic gradient estimator, we will use a version of the expected smoothness property \cite[Assumption 3.2]{khaled2020}.
We note that \cite{khaled2020} shows that variants of SGD including minibatching, importance sampling, gradient combinations and their combinations all satisfy expected smoothness.

\begin{assumption}(Expected Smoothness) \label{ass_new_gradients}
For all $k$, the gradient estimator $G(w_k)$ is unbiased:
\begin{equation}
     \mathbb{E}_k[G(w_k)] = \nabla F(w_k), \label{ass:1} 
\end{equation}
and there exist non-negative constants $A$, $B$ and $C$ (independent of $k$), such that: 
\begin{equation}
\mathbb{E}_k[\|G(w_k)\|^2] \leq 2A(F(w_k) -F_{\text{low}}) + B\|\nabla F(w_k)\|^2 + C. \label{ass:2} 
\end{equation}
\end{assumption}

To prove our result, we will need the below technical lemma, a generalization of \cite[Lemma 2]{khaled2020}, which corresponds to the case of $a_k \equiv a$ for all $k$.

\begin{lemma} \label{lem_weighting_sequence}
Let $\{r_k\}_{n \in \mathbb{N}}$, $\{d_k\}_{n \in \mathbb{N}}$ be positive-valued sequences, $\{a_k\}_{n \in \mathbb{N} \cup \{-1\}}$ be a non-increasing, positive-valued sequence and $L$, $A$ and $C$ be positive constants such that, for all $k \in \mathbb{N}$,
\begin{equation}
    \frac{a_k}{2} r_k \leq (1+a_k^2 L A)d_k -d_{k+1} + \frac{L C a^2_k}{2}. \label{lem4:1}
\end{equation}
Then, the following identity holds:
\begin{equation}
    \min_{0 \leq k \leq K-1} r_k \leq \frac{2 d_0}{a_{-1} K W_{K-1}}  + \frac{L C}{K W_{K-1}} \sum^{K-1}_{k=0} a_k W_k, \label{lem4:4}
\end{equation}
where $W_k$ is a weighting sequence defined by $W_{-1} = 1$ and $W_k = W_{k-1} \frac{a_k}{a_{k-1} (1+L A a_k^2)}$ for $k\geq 0$.
\end{lemma}
\begin{proof}
We begin by taking \eqref{lem4:1} and multiplying through by $\frac{W_k}{a_k}$ for each $k \in \mathbb{N}$. Thus, 
\begin{equation}
    \frac{1}{2} r_k W_k \leq (1+a_k^2 L A)d_k \frac{W_k}{a_k} - d_{k+1} \frac{W_k}{a_k} + \frac{L C a_k W_k}{2}.
\end{equation}
Noting that $\frac{W_k (1 + a_k^2 L A)}{a_k} = \frac{W_{k-1}}{a_{k-1}}$, we simplify: 
\begin{equation}
    \frac{1}{2} r_k W_k \leq d_k \frac{W_{k-1}}{a_{k-1}} - d_{k+1} \frac{W_k}{a_k} + \frac{L C a_k W_k}{2}.
\end{equation}
This allows us to apply a telescoping sum, adding together the first $K-1$ inequalities we receive: 
\begin{equation}
    \frac{1}{2} \sum^{K-1}_{k=0} r_k W_k \leq \frac{ d_0}{a_{-1}} - \frac{W_{K-1} d_K}{a_{K-1}} + \frac{L C}{2} \sum^{K-1}_{k=0} a_k W_k.
\end{equation}
Let $\hat{W} := \sum^{K-1}_{k=0} W_k$ and divide through by $\hat{W}$ on both sides: 
\begin{align}
    \frac{1}{2 \hat{W}} \sum^{K-1}_{k=0} r_k W_k &\leq \frac{ d_0}{a_{-1} \hat{W}} - \frac{W_{K-1} d_K}{a_{K-1} \hat{W}} + \frac{L C}{2 \hat{W}} \sum^{K-1}_{k=0} a_k W_k, \\
     \frac{1}{2 \hat{W}} \sum^{K-1}_{k=0} r_k W_k + \frac{W_{K-1} d_K}{a_{K-1} \hat{W}} &\leq \frac{d_0}{a_{-1} \hat{W}} + \frac{L C}{2 \hat{W}} \sum^{K-1}_{k=0} a_k W_k.
\end{align}
To simplify further we note: 
\begin{equation}
     \frac{1}{2} \min_{0 \leq k \leq K-1} r_k = \frac{1}{2 \hat{W}} \left(\min_{0 \leq k \leq K-1} r_k\right) \sum^{K-1}_{k=0} W_k \leq  \frac{1}{2 \hat{W}} \sum^{K-1}_{k=0} r_k W_k + \frac{w_{K-1} d_K}{a_{K-1} \hat{W}}, \label{lem4:2}
\end{equation}
where \eqref{lem4:2} holds as $W_{K-1}$, $d_K$, $a_{K-1}$ and $\hat{W}$ are all strictly positive. Hence:
\begin{equation}
    \frac{1}{2} \min_{0 \leq k \leq K-1} r_k \leq \frac{d_0}{a_{-1} \hat{W}} + \frac{L C}{2 \hat{W}} \sum^{K-1}_{k=0} a_k W_k \label{lem4:22}.
\end{equation}
We now note $\frac{a_k}{a_{k-1}} \leq 1$, as $\{a_k\}_{n \in \mathbb{N} \cup \{-1\}}$ is non-increasing. Furthermore, $\frac{1}{1+ L A a_k^2} \leq 1$, as $L,A\ \text{and}\ a_k > 0$. Therefore: 
\begin{equation}
     W_{k} = W_{k-1} \frac{a_k}{a_{k-1}} \frac{1}{(1+L A a_k^2)} \leq W_{k-1},
\end{equation}
and hence, $\{W_n\}_{\{n \in \mathbb{N} \cup \{-1\} \}}$ is non-increasing. Using this, we note:
\begin{equation}
    \frac{1}{\hat{W}} = \frac{1}{\sum^{K-1}_{k=0} W_k} \leq \frac{1}{\sum^{K-1}_{k=0} \min_{0 \leq k \leq K-1} W_k} = \frac{1}{K W_{K-1}}. \label{lem4:3}
\end{equation}
Substituting \eqref{lem4:3} into the right-hand side of \eqref{lem4:22} and multiplying by 2 provides our result.
\end{proof}

We are now ready to prove our main result.

\begin{theorem} \label{thm_main_convergence}
Suppose that Assumptions~\ref{ass_resilient_agg}, \ref{ass_new_smoothness} and \ref{ass_new_gradients} hold.
If we run \algref{alg_brsgd} with learning rate sequence that satisfies: $\{\alpha_k\}_{k \in \mathbb{N}}$ is non-increasing, $\alpha_0 \leq \frac{1 - \sin (\alpha)}{L B'}$, $\sum_{k=0}^{\infty} \alpha_k^2 < \infty$, and $\lim_{k \to \infty} \frac{1}{k \alpha_{k-1}} = 0$ 
(where $B' = B \revision{E}$; $L$, $B$ and \revision{$E$} are defined in Assumptions~\ref{ass_new_smoothness} and \ref{ass_new_gradients}, and \revision{\eqref{eq_resilient_agg_weaker}---a consequence of \assref{ass_resilient_agg}(2)---respectively, and $\alpha$ comes from \assref{ass_resilient_agg}(1)}), then
\begin{equation}
    \lim_{K \to \infty} \left(\min_{0\leq k \leq K-1} \mathbb{E}[\|\nabla F (w_k)\|]\right) = 0.
\end{equation}
\end{theorem}

\begin{proof}
We begin by applying \eqref{eq_lipschitz} and simplifying using $w_{k+1} = w_k - \alpha_k A_k$: 
\begin{align}
    F(w_{k+1}) &\leq F(w_k) + \langle \nabla F(w_k), w_{k+1}-w_k \rangle + \frac{L}{2} \|w_{k+1}-w_k\|^2, \\
    &= F(w_k) - \alpha_k \langle \nabla F(w_k) , A_k \rangle + \frac{L \alpha_k^2}{2}  \|A_k\|^2.
\end{align}
We then take the expectation of the above conditioned on $\mathcal{F}_k$ from our filtration:
\begin{align}
    \mathbb{E}_k[F(w_{k+1})] \leq F(w_k) - \alpha_k \langle \nabla F(w_k) , \mathbb{E}_k[A_k] \rangle + \frac{L \alpha_k^2}{2}  \mathbb{E}_k[\|A_k\|^2].
\end{align}
To simplify, we apply \revision{\assref{ass_resilient_agg}(2)} to get: 
\begin{equation}
    (1 - \sin(\alpha))\|\nabla F(w_k)\|^2 \leq \langle \nabla F(w_k),\mathbb{E}_k[A_k] \rangle, \label{thm52:1}
\end{equation}
and, for some constant $\revision{E}$,
\begin{equation}
    \mathbb{E}_k [\|A_k\|^2] \leq \revision{E} \mathbb{E}_k [\|G(w_k)\|^2]. \label{thm52:2}
\end{equation}
Applying \eqref{thm52:1} and \eqref{thm52:2}, and subtracting $F_{\text{low}}$ from both sides, we simplify:
\begin{align}
    \mathbb{E}_k[F(w_{k+1})] - F_{\text{low}} &\leq  F(w_k) - F_{\text{low}} - \alpha_k \langle \nabla F(w_k) , \mathbb{E}_k[A_k] \rangle + \frac{L \alpha_k^2}{2}  \mathbb{E}_k[\|A_k\|^2],\\
    &\leq F(w_k) - F_{\text{low}} - \alpha_k (1 - \sin(\alpha))\|\nabla F(w_k)\|^2 \\
    &\quad + \frac{L \alpha_k^2}{2} \revision{E} \mathbb{E}_k [\|G(w_k)\|^2].
\end{align}
We now simplify using \assref{ass_new_gradients}:
\begin{equation}
    \revision{E} \mathbb{E}_k [\|G(w_k)\|^2] \leq \revision{E} \left(2A(F(w_k) - F_{\text{low}}) + B \|\nabla F(w_k)\|^2 + C \right).
\end{equation}
Collecting terms, we will rewrite the upper bound for our final term as:
\begin{equation}
    \revision{E} \mathbb{E}_k [\|G(w_k)\|^2 \leq 2A'(F(w_k) - F_{\text{low}}) + B' \|\nabla F(w_k)\|^2 + C'.
\end{equation}
 
Thus: 
\begin{align}
     \mathbb{E}_k[F(w_{k+1})] - F_{\text{low}} &\leq (1+L \alpha_k^2 A')(F(w_k) - F_{\text{low}}) \\
     &\quad - \left(\alpha_k \left(1 - \sin(\alpha)\right) -\frac{L \alpha_k^2 B'}{2}\right)\|\nabla F(w_k)\|^2 + \frac{L \alpha_k^2 C'}{2}.
\end{align}

Taking total expectations and using the Tower Property: 
\begin{align}
    \mathbb{E}[F(w_{k+1}) - F_{\text{low}}] &\leq (1 + L \alpha_k^2 A') \mathbb{E}[(F(w_k)-F_{\text{low}})] \\
    &\quad- \left(\alpha_k \left(1 - \sin(\alpha)\right) -\frac{L \alpha_k^2 B'}{2}\right) \mathbb{E}[\|\nabla F(w_k)\|^2] + \frac{L \alpha_k^2 C'}{2}. \label{thm52:11}
\end{align}

Defining $\delta_{k} := \mathbb{E}[F(w_{k}) - F_{\text{low}}]$, \eqref{thm52:11} becomes:
\begin{equation}
    \alpha_k\left(1 - \sin(\alpha) - \frac{L B' \alpha_k}{2}\right) \mathbb{E}[\|\nabla F(w_k)\|^2] \leq (1 + L \alpha_k^2 A') \delta_k - \delta_{k+1} + \frac{L \alpha_k^2 C'}{2}.
\end{equation}

Furthermore, our requirements on \revision{the learning rate sequence $\alpha_k$} guarantee  $1-\sin(\alpha) - \frac{LB'\alpha_k}{2} \geq \frac{1-\sin(\alpha)}{2}$, hence we define $r_k := (1-\sin(\alpha)) \mathbb{E}[\|\nabla F(w_k)\|^2]$ and get:
\begin{equation}
    \frac{\alpha_k}{2} r_k \leq (1 + L \alpha_k^2 A') \delta_k - \delta_{k+1} + \frac{L \alpha_k^2 C'}{2}.
\end{equation}
We now apply \lemref{lem_weighting_sequence} to our problem using the following weighting sequence:
\begin{equation}
    W_k := 
    \begin{cases}
    1,& \text{if}\ k = -1,\\
    W_{k-1} \frac{\alpha_k}{\alpha_{k-1} (1+LA' \alpha_k^2)},    & \text{otherwise.}
    \end{cases}
\end{equation}

Hence:
\begin{equation}
    \min_{0 \leq k \leq K-1} r_k \leq \frac{2 \delta_0}{\alpha_{-1} K W_{K-1}} + \frac{LC'}{K  W_{K-1}} \sum^{K-1}_{k=0} \alpha_k W_k, \label{eq_after_lemma}
\end{equation}
where we define $\alpha_{-1} := \alpha_0$ for convenience. 

We now show $\lim_{K \to \infty} \left( \min_{0 \leq k \leq K-1} r_k \right) = 0$. To do this, we return to the definition of our weighting sequence and note:
\begin{equation}
      W_{k} = \frac{ \alpha_k}{\alpha_{-1}} \prod^{k}_{j=0} \frac{1}{1 + L A' \alpha_j^2}.
\end{equation}

Hence: 
\begin{align}
  \min_{0 \leq k \leq K-1} r_k &\leq \frac{2 \delta_0}{\alpha_{-1}} \left( \frac{\alpha_{-1}}{K \alpha_{K-1}} \prod^{K-1}_{k=0} (1 + L A' \alpha_k^2) \right) \nonumber \\
&\quad +  L C' \left( \frac{\alpha_{-1}}{K \alpha_{K-1}} \prod^{K-1}_{k=0} (1 + L A' \alpha_k^2) \right) \sum^{K-1}_{k=0} \left( \frac{\alpha_{k}^2}{\alpha_{-1}} \prod^{k}_{j=0} \frac{1}{1 + L A' \alpha_j^2} \right). \label{eq_tmp1}
\end{align}
In order to simplify, we will bound a pair of infinite products and a summation. We will first show that: 
\begin{equation}
    \prod^{\infty}_{k=0} (1 + L A' \alpha_k^2) < \infty.
\end{equation}
Recall that, for a sequence of positive real numbers $\{y_n\}_{n \in \mathbb{N}}$,  $\prod^{\infty}_{k=0} y_k$  converges if and only if $\sum^{\infty}_{k=0} \log(y_k)$ converges. From our specification of $\alpha_k$, we know $\sum^{\infty}_{k=0} \alpha_k^2 < \infty$. Let:
\begin{equation}
    Q := \sum^{\infty}_{k=0} \alpha_k^2.
\end{equation}
When $y > -1$, $\log(1+y) \leq y$. Hence: 
\begin{align}
\sum^{\infty}_{k=0} \log(1 + L A' \alpha_k^2) &\leq \sum^{\infty}_{k=0} L A' \alpha_k^2 \leq L A' Q. \label{eq_tmp_ref}
\end{align}
Therefore, $P := \prod^{\infty}_{k=0} (1 + L A' \alpha_k^2) < \infty$. Secondly, we note that, for all $k \in \mathbb{N}$:
\begin{equation}
    \prod^{k}_{j=0} \frac{1}{1 + L A' \alpha_k^2}\leq 1, \label{eq_tmp2}
\end{equation}
as, $L A' \alpha_k^2 > 0$, for all $k \in \mathbb{N}$. This, in turn, allows us to simplify another summation. Specifically:  
\begin{align}
  \sum^{\infty}_{k=0} \left( \frac{\alpha_{k}^2}{\alpha_{-1}} \prod^{k}_{j=0} \frac{1}{1 + L A' \alpha_j^2} \right) &\leq \frac{1}{\alpha_{-1}} \sum^{\infty}_{k=0} \alpha_k^2 = \frac{Q}{\alpha_{-1}}.
\end{align}
We now take the limit of both sides of \eqref{eq_tmp1} and recover our result. 
\begin{align}
     \lim_{K \to \infty} &\left(\min_{0 \leq k \leq K-1} r_k\right) \nonumber \\
     &\leq \lim_{K \to \infty} \left( \frac{2  \delta_0}{\alpha_{-1}} \left( \frac{\alpha_{-1}}{K \alpha_{K-1}} \prod^{K-1}_{k=0} (1 + L A' \alpha_k^2) \right) \right) \\
    &\quad+  \lim_{K \to \infty} \left(L C' \left( \frac{\alpha_{-1}}{K \alpha_{K-1}} \prod^{K-1}_{k=0} (1 + L A' \alpha_k^2) \right) \sum^{K-1}_{k=0} \left( \frac{\alpha_k^2}{\alpha_{-1}} \prod^{k}_{j=0} \frac{1}{1 + \alpha_j^2 L A'} \right) \right), \\
	&\leq \revision{2\delta_0 \left(\lim_{K\to\infty} \frac{1}{K\alpha_{K-1}}\right) \prod_{k=0}^{\infty}(1+LA' \alpha_k^2)} \nonumber \\
	&\quad \revision{+ L C' \alpha_{-1} \left(\lim_{K\to\infty} \frac{1}{K\alpha_{K-1}}\right) \left(\prod_{k=0}^{\infty}(1+LA'\alpha_k^2)\right) \left(\sum^{\infty}_{k=0} \left( \frac{\alpha_{k}^2}{\alpha_{-1}} \prod^{k}_{j=0} \frac{1}{1 + L A' \alpha_j^2} \right)\right)}, \\
	&= \revision{(2\delta_0 P + L C' P Q) \lim_{K\to\infty} \frac{1}{K\alpha_{K-1}}}, \label{eq_nearly_done}
\end{align}
\revision{where $P$ is defined after \eqref{eq_tmp_ref}}. By assumption, we have $\lim_{K \to \infty} \frac{1}{K \alpha_{K-1}} = 0 $, therefore:
\begin{equation}
    \lim_{K \to \infty} \left(\min_{0 \leq k \leq K-1} (1-\sin(\alpha)) \mathbb{E}[\|\nabla F(w_k)\|^2]\right) = 0.
\end{equation}
Our result then follows from Jensen's inequality, $\mathbb{E}[\|\nabla F(w)\|]^2 \leq \mathbb{E}[\|\nabla F(w)\|^2]$.
\end{proof}

We note that our assumptions on the learning rate sequence $\{\alpha_k\}_k$ allows for sequences which decay as $\alpha_k \sim k^{-p}$ for any $p\in(1/2,1)$.

\revision{
\subsection{Iterate Selection} \label{sec_iterate_selection}
The above result shows that there is a subsequence of iterates which converges in expectation.
We now give a probabilistic procedure to select a single iterate with small expected gradient, inspired by the analysis in \cite[Theorem 6.1]{Lan2019}.
In this context, we assume that \algref{alg_brsgd} has been run for $K$ iterations, and we randomly select an iterate from $w_0,\ldots,w_{K-1}$ according to a specific probability distribution depending on the learning rate sequence $\alpha_k$.
However, compared to \cite[Corollary 6.1]{Lan2019}, the choice of $\alpha_k$ does not depend on $K$, and so \algref{alg_brsgd} can always be continued from its previous endpoint if the desired accuracy is not achieved.

This analysis requires only small modifications of \lemref{lem_weighting_sequence} and \thmref{thm_main_convergence}, which we present here.

\begin{lemma} \label{lem_weighting_sequence_random}
Suppose the assumptions of \lemref{lem_weighting_sequence} hold, including the definition of the weighting sequence $W_k$.
If, for any $K\geq 0$, we define the random variable $R_K\in\{0,\ldots,K-1\}$ by
\begin{align}
	\mathbb{P}(R_K=k) := \frac{W_{k}}{\sum^{K-1}_{i=0} W_i}, \quad k=0,...,K-1, \label{eq_weighting_iterate}
\end{align}
then
\begin{equation}
    \mathbb{E}[r_{R_K}] \leq \frac{2 d_0}{a_{-1} K W_{K-1}}  + \frac{L C}{K W_{K-1}} \sum^{K-1}_{k=0} a_k W_k.
\end{equation}
\end{lemma}
\begin{proof}
The proof of this result is identical to that of \lemref{lem_weighting_sequence}, except that \eqref{lem4:2} is replaced by
\begin{align}
	\frac{1}{2} \mathbb{E}[r_{R_K}] = \frac{1}{2 \hat{W}} \sum^{K-1}_{k=0} r_k W_k \leq  \frac{1}{2 \hat{W}} \sum^{K-1}_{k=0} r_k W_k + \frac{w_{K-1} d_K}{a_{K-1} \hat{W}},
\end{align}
from which we conclude
\begin{align}
	\frac{1}{2} \mathbb{E}[r_{R_K}] \leq \frac{d_0}{a_{-1} \hat{W}} + \frac{L C}{2 \hat{W}} \sum^{K-1}_{k=0} a_k W_k,
\end{align}
in place of \eqref{lem4:22}.
\end{proof}

\begin{corollary} \label{cor_convergence_random_iterate}
Suppose that the assumptions of \thmref{thm_main_convergence} hold, and for any $K\geq 0$ we define the random variable $R_K$ as per \eqref{eq_weighting_iterate}.
Then
\begin{equation}
    \lim_{K \to \infty} \mathbb{E}[\|\nabla F (w_{R_K})\|] = 0.
\end{equation}
\end{corollary}
\begin{proof}
	The proof of this result is identical to that of \thmref{thm_main_convergence}, but we replace \eqref{eq_after_lemma} with 
	\begin{align}
		\mathbb{E}[r_{R_K}] \leq \frac{2 \delta_0}{\alpha_{-1} K W_{K-1}} + \frac{LC'}{K  W_{K-1}} \sum^{K-1}_{k=0} \alpha_k W_k,
	\end{align}
	which follows from \lemref{lem_weighting_sequence_random}.
	Hence instead of \eqref{eq_nearly_done} we reach
	\begin{align}
		\lim_{K \to \infty} \mathbb{E}[r_{R_K}] &\leq (2\delta_0 P + L C' P Q) \lim_{K\to\infty} \frac{1}{K\alpha_{K-1}} = 0, 
	\end{align}
	from which the result follows by Jensen's inequality, $(1-\sin(\alpha)) \mathbb{E}[\|\nabla F(w_{R_K})\|]^2 \leq r_{R_K}$.
\end{proof}
} 

\revision{
\section{Convergence Rate} \label{sec_complexity}
The previous section gives a convergence analysis for \algref{alg_brsgd} under general assumptions on the decreasing learning rate sequence $\alpha_k$.
We now specialize these results to give a convergence rate for the case $\alpha_k \sim k^{-p}$ for $p\in(1/2,1)$.

\begin{corollary} \label{cor_complexity}
	Suppose the assumptions of \thmref{thm_main_convergence} hold, and the learning rate sequence is given by
	\begin{align}
		\alpha_k = \frac{1-\sin(\alpha)}{L B' (k+1)^p}, \qquad k=0,1,2,\ldots,
	\end{align}
	for some $p\in(1/2,1)$.
	Then
	\begin{align}
		\min_{0 \leq k \leq K-1} \mathbb{E}[\|\grad F(w_k)\|] &\leq \left(\frac{e^{L A' Q_K} \left(2\delta_0  + L C' Q_K\right)}{(1-\sin(\alpha))K \alpha_{K-1}}\right)^{1/2}, \label{eq_complexity1}
	\end{align}
	where $Q_K := \sum_{k=0}^{K-1} \alpha_k^2$, $A':=AE$, $C':=CE$ (with $A$ and $C$ from \assref{ass_new_gradients} and $E$ from \eqref{eq_resilient_agg_weaker}), and $\delta_{0} := F(w_0)-F_{\text{low}}$.
	
	In particular, this means that
	\begin{align}
		\min_{0 \leq k \leq K-1} \mathbb{E}[\|\grad F(w_k)\|] = \bigO\left(e^{A'/(2 L (B')^2)} \left(2\delta_0+\frac{C'}{L (B')^2}\right)^{1/2} (L B')^{1/2} \frac{1}{K^{(1-p)/2}}\right),
	\end{align}
	as $K\to\infty$ (considering the dependency in terms of $K$, $\delta_0$, $A'$, $B'$, $C'$, and $L$).
\end{corollary}
\begin{proof}
	We continue from \eqref{eq_tmp1} in the proof of \thmref{thm_main_convergence}, where after applying \eqref{eq_tmp2} we have
	\begin{align}
		\min_{0 \leq k \leq K-1} r_k &\leq \frac{2 \delta_0}{K \alpha_{K-1}} \prod^{K-1}_{k=0} (1 + L A' \alpha_k^2)  + \left( \frac{L C'}{K \alpha_{K-1}} \prod^{K-1}_{k=0} (1 + L A' \alpha_k^2) \right) \sum^{K-1}_{k=0}  \alpha_{k}^2,
	\end{align}
	where $r_k := (1-\sin(\alpha)) \mathbb{E}[\|\nabla F(w_k)\|^2]$.
	By the same reasoning leading to \eqref{eq_tmp_ref} we get
	\begin{align}
		\sum_{k=0}^{K-1} \log(1+LA'\alpha_k^2) \leq L A' Q_K,
	\end{align}
	and so $\prod^{K-1}_{k=0} (1 + L A' \alpha_k^2) \leq e^{L A' Q_K}$, giving
	\begin{align}
		\min_{0 \leq k \leq K-1} r_k &\leq \frac{e^{L A' Q_K} \left(2\delta_0  + L C' Q_K\right)}{K \alpha_{K-1}}.
	\end{align}
	Jensen's inequality gives $\mathbb{E}[\|\nabla F(w)\|]^2 \leq \mathbb{E}[\|\nabla F(w)\|^2]$, and we recover \eqref{eq_complexity1}.

	Now since $\alpha_k = \alpha_0 (k+1)^{-p}$ we have
	\begin{align}
		Q_K \leq \alpha_0^2 + \int_{1}^{K-1} \alpha_0^2 k^{-2p} dk \leq \alpha_0^2 + \int_{1}^{K} \alpha_0^2 k^{-2p} dk = \frac{\alpha_0^2}{2p-1}\left(2p-2 + \frac{1}{K^{2p-1}}\right).
	\end{align}
	So for $K$ large, we have $Q_K = \bigO(\alpha_0^2) = \bigO\left(\frac{1}{L^2 (B')^2}\right)$ and $K \alpha_{K-1} = \bigO\left(\frac{K^{1-p}}{L B'}\right)$, which gives second result.
\end{proof}

\begin{remark}
	Identical results to \corref{cor_complexity} hold when considering $\mathbb{E}[\|\grad F(w_{R_K})\|]$ instead of $\min_{k\leq K} \mathbb{E}[\|\grad F(w_k)\|]$, following the reasoning in \secref{sec_iterate_selection}.
\end{remark}

Our convergence rate result \corref{cor_complexity} essentially gives $\min_{k\leq K} \mathbb{E}[\|\grad F(w_k)\|] = \bigO\left(\frac{1}{K^{(1-p)/2}}\right)$.
In the best case, $p\to 1/2$, this convergence rate approaches $\bigO(K^{-1/4})$, or equivalently requiring $K = \bigO(\epsilon^{-4})$ to achieve optimality  level $\mathbb{E}[\|\grad F(w_k)\|] \leq \epsilon$ (or alternatively $\mathbb{E}[\|\grad F(w_{R_K})\|] \leq \epsilon$).
This matches the standard worst-case complexity of stochastic gradient descent for nonconvex functions \cite{pmlr-v119-drori20a}.
} 

\section{Conclusion} \label{sec_conclusion}
Having algorithms for \revision{distributed} learning in the presence of Byzantine adversaries is an important part of improving the utility of \revision{distributed} learning.
In this work we presented a simplified analysis of Byzantine resilient SGD (BRSGD) as developed in \cite{blanchard2017byzantinetolerant}, proving convergence in expectation \revision{and corresponding convergence rates} under more realistic assumptions on the objective function and the (non-corrupted) stochastic gradient estimators.
Since BRSGD is a generic algorithm, allowing the use of any aggregation function satisfying \assref{ass_resilient_agg}, our analysis applies to all the specific choices given in \cite{blanchard2017byzantinetolerant,xie2018generalized,pmlr-v80-mhamdi18a}.
Directions for future work include extending our analysis to more flexible learning rate regimes and other Byzantine resilient learning algorithms not based on the framework from \cite{blanchard2017byzantinetolerant}, such as those in \cite{pmlr-v80-yin18a,pmlr-v97-yin19a}.

\addcontentsline{toc}{section}{References} 
\bibliographystyle{siam}
\bibliography{refs2} 


\end{document}